\documentclass{article}

\usepackage{microtype}
\usepackage{graphicx}
\usepackage{subfigure}
\usepackage{booktabs} 

\usepackage{hyperref}

\usepackage[accepted]{icml2025}

\usepackage{amsmath}
\usepackage{amssymb}
\usepackage{mathtools}
\usepackage{amsthm}

\usepackage[capitalize,noabbrev]{cleveref}

\theoremstyle{plain}
\newtheorem{theorem}{Theorem}[section]

\newtheorem{lemma}[theorem]{Lemma}

\theoremstyle{definition}
\newtheorem{definition}[theorem]{Definition}

\theoremstyle{remark}

\usepackage[textsize=tiny]{todonotes}

\icmltitlerunning{FOCUS: First Order Concentrated Updating Scheme}

\begin{document}

\twocolumn[
\icmltitle{FOCUS: First Order Concentrated Updating Scheme}



\icmlsetsymbol{equal}{*}

\begin{icmlauthorlist}
\icmlauthor{Yizhou Liu}{yyy}
\icmlauthor{Ziming Liu}{yyy}
\icmlauthor{Jeff Gore}{yyy}
\end{icmlauthorlist}

\icmlaffiliation{yyy}{Massachusetts Institute of Technology, Cambridge, MA 02139, USA}

\icmlcorrespondingauthor{Yizhou Liu}{liuyz@mit.edu}

\icmlkeywords{Optimization, Training Dynamics, LLM}

\vskip 0.3in
]



\printAffiliationsAndNotice{}  

\begin{abstract}
Large language models (LLMs) demonstrate remarkable performance, and improving their pre-training process appears to be key to enhancing their capabilities further. Based on the documented success of Adam, learning rate decay, and weight decay, we hypothesize that the pre-training loss landscape features a narrowing valley structure. Through experiments with synthetic loss functions, we discover that when gradient query noise is high relative to the valley's sharpness, Adam's performance falls behind that of Signum because Adam reduces the effective step size too drastically. This observation led us to develop FOCUS, an optimizer that enhances Signum by incorporating attraction toward moving averaged parameters, allowing it to handle noise better while maintaining larger step sizes. In training GPT-2, FOCUS proves to be more stable than Signum and faster than Adam. These results suggest that gradient noise may be an underappreciated limiting factor in LLM training, and FOCUS offers promising solutions.
\end{abstract}

\section{Introduction}\label{sec:intro}
\begin{figure*}[ht]
\vskip 0.2in
\begin{center}
\centerline{\includegraphics{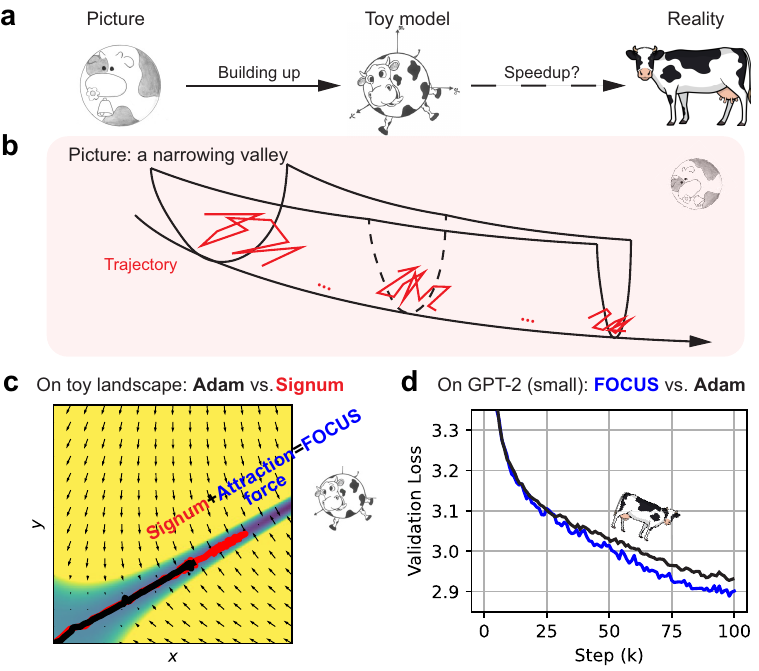}}
\caption{Insights from toy models lead to practical speedup in LLM pre-training. (a) Our philosophy that a simple and self-consistent picture provides a basis of thinking and testable hypotheses is not necessarily correct. Interactions between testing the real world and updating our picture keep us moving forward. (b) Narrowing valleys are assumed to be a key structure in LLM training loss. (c) The toy model explains that Adam can be slow with a large gradient noise (\cref{sec:method}). (d) By adding self-attraction to Signum, FOCUS is proposed to handle sharp landscapes with large gradient stochasticity and achieve practical speedups (\cref{sec:exp}).}
\label{fig:outline}
\end{center}
\vskip -0.2in
\end{figure*}

The emergence of large language models (LLMs) represents a transformative advancement in artificial intelligence. These systems have demonstrated remarkable capabilities that extend far beyond their initial purpose of processing and generating human language. From engaging in nuanced conversations and crafting creative content \cite{brown2020gpt3,openai2023gpt4,qin2023bardvisual} to solving complex mathematical problems \cite{lewkowycz2022solving,taylor2022galactica,wolfram2023integration,trinh2024alphageometry} and assisting with software development \cite{chen2021codex,github2022copilot}, LLMs continue to expand the boundaries of what artificial intelligence can achieve. 

To further advance the capabilities of LLMs, improving the pre-training is a critical research direction \cite{kaplan2020scaling,hoffmann2022chinchilla,liu2023sophia}. Better optimizers are desired to reduce time, cost, model size, etc. for reaching the same loss or to achieve lower losses given similar budgets. To this end, it is essential to understand the loss landscapes in LLM pre-training and design optimizers accordingly.

To gain insights into the training dynamics, we take the physicists' approach of simplifying the real world into a picture and a toy model (\cref{fig:outline}a).  
A picture describes only the important features and ignores all other details. A toy model then specifies the picture with quantification. We can gain insights and generate testable hypotheses through a systematic study of the toy model. Then, we can apply our insights to engineer the real systems or we may fail and update our picture.

In this paper, we start in \cref{sec:picture} to build a picture for pre-training landscapes, i.e., a narrowing valley (\cref{fig:outline}b), based on previous results \cite{wei2019noise,liu2023sophia,wen2024understanding}. Depending upon the topology of the loss landscape, different optimizers may respond differently.
We next construct loss landscapes and study the behaviors of different optimizers in \cref{sec:toy}. We find when the noise in gradient is too large (see an example in \cref{fig:outline}c), Adam~\cite{kingma2014adammethodstochasticoptimization,loshchilov2018decoupled} performs worse than Signum~\cite{bernstein2018signsgd} which takes constant size steps based on the sign of the momentum (smoothed gradient). Adam decreases its effective step size too much when the gradient is too noisy while Signum fixes the step size. Inspired by this observation, we designed the FOCUS optimizer by adding an attraction force towards the moving averaged parameters to Signum, so that FOCUS can squeeze into valleys with fixed step sizes (see \cref{sec:focus}).
In GPT-2 (small) pre-training, FOCUS is more stable than Signum and Adam with similar learning rates. FOCUS is favored compared to Adam (\cref{fig:outline}d), achieving a twice speedup compared to Adam from the literature in terms of training time (see details in \cref{sec:exp}). We provide a convergence analysis of FOCUS in \cref{sec:theory}, discuss related works in \cref{sec:relate}, and conclude with physics intuitions in \cref{sec:disscuss}.

Our results contribute to both \textbf{scientific understanding} of training dynamics and \textbf{practical speedup} in LLM pre-training. We propose a minimal model with sharp valley and gradient noise to study training dynamics. We obtain a better understanding of Adam's advantage in handling the sharpness and its limitation when gradient noise is large. Finally, we propose FOCUS using attraction force to squeeze into valleys without decreasing effective step size much, which gains actual speedup in training GPT-2 (small). We anticipate that FOCUS will provide new insights and directions for designing optimizers that will inspire a variety of new optimizers.

\section{Methods}\label{sec:method}

\subsection{Picture of loss landscapes}\label{sec:picture}

In this part, we first describe the narrowing valley picture depicted in \cref{fig:outline}b and then explain why this picture is relevant to the actual pre-training losses of LLMs.

A valley structure means heterogeneous curvatures across parameter dimensions. There are many sharp and flat dimensions, where the sharp ones form the walls of the valley and the valley extends along the flat ones (\cref{fig:outline}b). Going down along the valley, one obtains lower loss. However, the sharp directions may be sharper or more. For an optimizer, it is easy to find the bottom of one valley along the sharp directions, but it may be challenging to stay at the bottom and follow the flat directions to move forward. We summarize the picture of the loss landscape as a ``narrowing valley".

The evidence for heterogeneous curvatures is explained as follows. 
A direct study of Hessian of GPT-2 supports our picture \cite{liu2023sophia}.
Theory suggests that large uncertainty variation in data distribution can lead to heterogeneous curvatures \cite{wen2024understanding}, which is likely to be true in language as the same phrase can naturally have many different continuations. 
A piece of side evidence is that Adam and its variants were shown to outperform stochastic gradient descent (SGD) in LLM training \cite{zhao2024deconstructing}.
Adam is better at dealing with heterogeneous directions since it normalizes the updates in different parameter directions to be of the same order of magnitude. However, SGD can be extremely slow along the flat directions if it wants convergence along the sharp directions \cite{liu2023sophia}.

Another important feature is that the valley is getting narrower and narrower. We conjecture this property due to the success of some learning rate schedulers \cite{loshchilov2017sgdrstochasticgradientdescent,jin2023rethinkinglearningratetuning,Subramanian2024}: decreasing learning rate at the late stage of training can help to reach a lower loss. Due to the discretized nature of optimizers, a smaller learning rate is needed for convergence given larger curvature \cite{liu2023sophia}. For optimizers like Adam, a smaller learning rate helps to be close to the bottom, especially at places with large curvatures. In the continuous limit, learning rates also play the role of temperature \cite{shi2020learningratesschrodingeroperators,Liu2023quantumspeedups} and decreasing temperature makes ``bouncing particles" easier to squeeze into sharper valleys. According to different heuristics and theories, the fact that lowering learning rates works supports the narrowing valley picture.

Compared to the conventional picture of the loss landscapes, one important difference is that we do not draw the minimum (\cref{fig:outline}b). We are assuming that the true minimum is far from being reached. This is inspired by the fact that lowering weight decay can still result in larger parameter norms in current LLM pre-training \cite{andriushchenko2024why,wang2024setadamwsweightdecay}. This feature also indicates that the valley should be narrower when moving forward. Otherwise, the optimizer should reach the minimum easily even when far away.

Besides the key ingredients introduced, there are several factors worth mentioning. 
First, the loss function is defined by the distribution of data. Due to the finite batch size, the gradient obtained can be thought of as the gradient of the landscape with some noise. Second, the valley structure is one elementary ingredient. The landscape can be highly non-convex by having many valleys. As explained in Introduction (\cref{sec:intro}), our picture keeps the simplified fundamental structure and updates details when the original one does not suffice.
Finally, the picture merits study in its own right, independent of its applications to optimization. The picture offers a base of thinking and can help to generate all kinds of techniques and testable hypotheses. The picture is also open to modifications based on feedback from experiments on real systems.

\subsection{Toy model}\label{sec:toy}
\begin{figure*}[ht]
\vskip 0.2in
\begin{center}
\centerline{\includegraphics{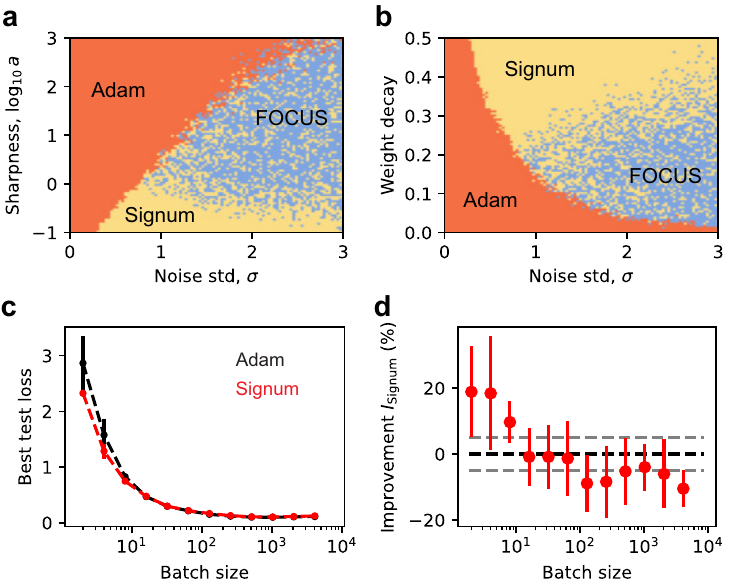}}
\caption{Signum outperforms Adam when gradient stochasticity is relatively large, and FOCUS further improves Signum when the valley-like landscape is sharp. (a) Signum outperforms Adam when gradient noise is large compared to sharpness. Orange pixels refer to conditions Adam is better and yellow parts mean Signum is better. FOCUS is even better than Signum after increasing sharpness. Blue pixels refer to conditions FOCUS is the best. (b) Increasing weight decay helps Signum and FOCUS against Adam. Yet the advantage of self-attraction is lost when weight decay is too large. (c and d) Gradient noise is large when batch size is small in practice. For MNIST classification, We find that increasing batch size (decreasing noise) can lead to a transition of the optimal optimizer from Signum to Adam, showing the insights from toy models are relevant to reality. Grey dashed lines in (d) highlight $\pm 5$ \%. The error bars represent standard deviations. Experiment details are in \cref{app:models}.}
\label{fig:toy}
\end{center}
\vskip -0.2in
\end{figure*}

In this section, we present our toy landscapes in the spirit of ``narrowing valleys" and then our insights gained from the experiments. 

Our toy loss function defined on $x$ and $y$ (can be thought of as model parameters) is given by
\begin{align}
    L(x,y) = \frac{a}{2}u^2v^2 - cu,
    \label{eq:toyloss}
\end{align}
where $u$ and $v$ depends on $x$ and $y$ via a rotation, $a$ and $c$ are two hyperparameters controlling the loss. 
The valley direction is in general not aligned with the model parameters.
\cref{fig:outline}c depicts an example of the loss where $u$ is rotated counter-clockwise from $x$ by $\pi/6$.
The term $- cu$ ($c>0$) ensures that $+u$ is the direction along the valley to lower loss values.
The curvature of the orthogonal direction $v$ is given by $au^2$ which is getting larger when moving along the valley (going to greater $u$).  

Upon building the landscape, we start our exploration with the small motivating question--why can Signum be comparable to Adam in LLM pre-training \cite{zhao2024deconstructing}?
Adam updates its parameters by $-\eta \hat{m}_t / \sqrt{\hat{v}_t}$ at step $t$ \cite{kingma2014adammethodstochasticoptimization}, where $\eta$ is the learning rate, $\hat{m}_t$ is the centered exponential moving average (EMA) of gradient, and $\hat{v}_t$ is the centered EMA of the squared gradient.
Specifically, we have
\begin{align}
    m_t = m_{t-1}\beta_1 + (1-\beta_1)g_t,
\end{align}
where $g_t$ (a vector) is the gradient queried at step $t$ and $0<\beta_1<1$ is a hyperparameter controlling how fast old gradients are forgotten. Similarly, we have
\begin{align}
    v_t = v_{t-1}\beta_2 + (1-\beta_2)g_t^2.
\end{align}
All operations on vectors are element-wise. The centered averages $\hat{m}_t$ and $\hat{v}_t$ are normalized from $m_t$ and $v_t$, respectively, since the initialization $m_0=v_0=0$ introduces bias towards $0$ \cite{kingma2014adammethodstochasticoptimization}.
The idea is that along sharp directions, corresponding elements in $g_t$ may bounce around $0$ such that $\hat{m}_t / \sqrt{\hat{v}_t}$ is small and the effective step size is small along those directions. While for flat directions, corresponding elements in $g_t$ are consistent across different $t$, then $\hat{m}_t / \sqrt{\hat{v}_t}$ can be large or even close to $1$. By construction, Adam is good at squeezing into narrowing valleys. Throughout the paper, when we need to add weight decay, we follow the method in AdamW \cite{loshchilov2018decoupled}.

Signum can be considered a simplified version of Adam that only uses momentum. Signum updates the parameters by $-\eta \mathrm{sign}(\hat{m}_t)$ at step $t$. Signum uses fixed step size, losing the ability to adjust step size based on sharpness or the degree of query noise. Consequently, we would expect Signum to oscillate more violently and perform much worse than Adam in sharp valleys.

However, this expectation is not always true \cite{zhao2024deconstructing}. We compare Signum and Adam on our toy landscape to get possible explanations for the phenomenon.
First of all, we indeed find that Signum can be faster than Adam and can squeeze into deeper places along the valley under the same hyperparameter settings (\cref{fig:outline}c). We discovered this case when examining the effect of noise in gradient estimation, e.g., due to finite batch sizes. We model this stochasticity by multiplying the true gradient of the loss \cref{eq:toyloss} by a random variable with mean $1$ and standard deviation $\sigma$ when the gradient is queried. We obtain a hypothesis that large noise may change the competition between Adam and Signum.

Next, we compare Adam and Signum more systematically and fairly. As explained in \citet{zhao2024deconstructing}, different optimizers can have different optimal learning rates and we need to compare their best performances with their own best hyperparameters. Given a loss function, we scan $20$ learning rates between $10^{-3}$ and $1.0$ (uniform in log scale) for each optimizer. We run $50$ optimization tests for each learning rate where each test starts near the origin and runs for $10^3$ steps. The optimal learning rate is selected via the mean of final losses, typically around $10^{-2}$. 

In \cref{fig:toy}a, we examine the best performance of different optimizers at different valley sharpness, $a$, and noise standard deviation, $\sigma$ ($c=0.1$ is set to be constant). The orange region refers to the parameters where Adam is the best and the yellow region for those where Signum is the best. We find that given sharpness $a$ there is a critical noise strength $\sigma_c$ beyond which Signum will become better than Adam. Moreover, this critical noise strength $\sigma_c$ increases when sharpness $a$ increases. Intuitively, both large sharpness (violent oscillation) and large noise lead to a small signal-to-noise ratio $\hat{m}_t / \sqrt{\hat{v}_t}$. When a small signal-to-noise ratio is mainly due to sharpness, the small effective step size helps to get into the valley, and Adam is therefore faster. However, when a small signal-to-noise ratio is mainly due to noise, Adam decreases the effective step size too much and becomes slower. A small step size also makes Adam stop at a place closer to the origin given the same weight decay. We conclude that Signum outperforms Adam when the gradient noise is large compared to the sharpness of the valley. 

For completeness, we should also compare the performance of optimizers at their best weight decays (another hyperparameter they all have). On our toy landscape, the smaller the weight decay is, the better the performance will be, which is true for all optimizers tested. However, training LLM usually does not set zero weight decay for many reasons \cite{andriushchenko2024why}. Some non-zero small weight decay is needed for other aspects not described by our toy model. And the smallest possible weight decay will be the optimal one for all optimizers in terms of final loss. We then can choose that value for weight decay and compare different optimizers with their own best learning rates. Since the smallest possible weight decay is some unknown value, we test different weight decay values (\cref{fig:toy}b). We find no matter which weight decay is used, small noise prefers Adam and large noise prefers Signum as before. And at the same degree of stochasticity, Signum is favored at larger weight decay. Large weight decay may help Signum to be stable and stay closely in the valley.
We conclude from our toy model experiments that at various conditions, Signum outperforms Adam when the noise in gradients is relatively large. 

Finally, we try to test our findings from the toy model on real optimization problems. In particular, the noise of gradient queries in the toy model is artificial yet is believed to be related to finite batch size in real problems. We then would like to check whether in reality there is a transition from Signum outperforming Adam to Adam outperforming Signum when increasing batch sizes (decreasing gradient stochasticity). We use a six-layer multilayer perceptron (see \cref{app:mlp} for details) to do MNIST classification \cite{lecun1998mnist}. Given the batch size, we find the optimal learning rate for each optimizer by scanning $20$ learning rates from $10^{-4}$ to $1.0$ uniformly in log scale. For each learning rate, we run three replicates and each run has $400$ steps to ensure saturation. Each learning rate has a mean final loss over replicates and the best test loss is chosen from these mean final losses of different learning rates. The best test loss of Adam indeed is larger than that of Signum at small batch sizes but becomes smaller than that of Signum after the batch size gets large (\cref{fig:toy}c). This trend agrees with our insights obtained from the toy model. We define the improvement of Signum as
\begin{equation}
    I_{\rm Signum} = \frac{L_{\rm Adam} - L_{\rm Signum}}{L_{\rm Adam}},
\end{equation}
where $L_{\rm Adam}$ and $L_{\rm Signum}$ are best losses of Adam and Signum, respectively. The improvement $I_{\rm Signum}$ captures the transition from Signum being better to Adam being better when increasing batch sizes (\cref{fig:toy}c). We conclude that our finding on toy landscapes about the competition between Adam and Signum is relevant to realistic machine learning problems.

\subsection{FOCUS optimizer}\label{sec:focus}
\begin{algorithm}[tb]
   \caption{FOCUS (All operations on vectors are element-wise and $\beta_2^t$ means $\beta_2$ to the power $t$)}
   \label{alg:focus}
\begin{algorithmic}
   \STATE {\bfseries Input:} Initial parameter vector $\theta_1$, learning rate $\{\eta_t\}_{t=1}^T$, hyperparameters $\beta_1$, $\beta_2$, $\gamma$, $\omega$
   \STATE Set $m_0 = 0$, $\overline{\theta}_0=0$
   \FOR{$t=1$ {\bfseries to} $T$}
    \STATE Compute minibach loss $L_t(\theta_t)$
    \STATE Obtain gradient $g_t = \nabla L_t(\theta_t)$
    \STATE $m_t = \beta_1 m_{t-1} + (1-\beta_1) g_t$ (Update biased momentum)
    \STATE $\overline{\theta}_t = \beta_2 \overline{\theta}_{t-1} + (1-\beta_2) \theta_t$ (Update biased average parameters)
    \STATE $\hat{\theta}_t = \overline{\theta}_t / (1 - \beta_2^t)$ (Compute bias-corrected average parameters)
    \STATE $\theta_t = \theta_t - \eta_t \omega \hat{\theta}_t$ (Apply weight decay)
    \STATE $\theta_{t+1} = \theta_t - \eta_t (\mathrm{sign}(m_t) + \gamma \mathrm{sign}(\theta_t - \hat{\theta}_t))$
   \ENDFOR
\end{algorithmic}
\end{algorithm}

Good insights into the training dynamics should lead to practical speedup. To this end, we try to design new optimizers based on what we learned from the toy model experiments.
On the one hand, we want the new optimizer not to decrease the effective step size like Adam such that it can work well when gradient noise is large. On the other hand, we want the optimizer to be able to squeeze into the sharp valley. We then try to add some new mechanisms to Signum for better dealing with sharpness.
Imagining the picture \cref{fig:outline}b, we want the trajectory to be more compact. One way to realize that is to add an attraction force between the ``bouncing particles", i.e., parameter values at different time steps.

One realization of adding an attraction force term to Signum is our FOCUS (\textbf{F}irst \textbf{O}rder \textbf{C}oncentrated \textbf{U}pdating \textbf{S}cheme) optimizer in \cref{alg:focus}. The optimizer queries $0$th and $1$st order information. We use the word ``concentrated" to highlight the attraction force term. Upon getting gradient $g_t$, FOCUS updates the EMA of gradient, $m_t$, which is called momentum, as well as the EMA of parameters, $\overline{\theta}_t$.
The hyperparameters $\beta_1,\beta_2\in [0,1)$ are the decay rates of the EMAs of gradients and parameters, respectively.
Since we initialize $\overline{\theta}_0=0$, we need normalization $\hat{\theta}_t = \overline{\theta}_t / (1 - \beta_2^t)$ to have an unbiased estimation of the moving average of parameters \cite{kingma2014adammethodstochasticoptimization}. We do not compute bias-corrected momentum since we only need the sign later and normalization does not change the sign. We update the parameters by $- \eta_t \mathrm{sign}(m_t)$ which follows Signum and by $- \eta_t\gamma \mathrm{sign}(\theta_t - \hat{\theta}_t)$ which is the attraction force term. The hyperparameter $\gamma\in [0,1)$ controls the strength of the attraction.
The attraction moves parameters towards the center of the trajectory or bottom of the valley (\cref{fig:outline}b). In principle, we can make $\gamma$ negative and have repulsive forces, which is not of interest in this paper. We add weight decay with a similar method as AdamW \cite{loshchilov2018decoupled} and $\omega$ is the vector encoding weight decay (not all parameters have non-zero weight decay).

We next study the performance of FOCUS on our toy landscapes.
We follow the same procedure to find the best performance over different learning rates of FOCUS ($\gamma=0.2$) as Adam and Signum in \cref{fig:toy}. We use the blue region to denote the conditions where FOCUS is the best optimizer. We find that FOCUS outperforms Signum when sharpness is large with high probability (\cref{fig:toy}a). The checkered pattern between blue and yellow suggests that FOCUS and Signum reach similar depths, but FOCUS appears more stable staying deep in the valley. FOCUS cannot be better than Signum if weight decay is too large (\cref{fig:toy}b). A large weight decay helps Signum to stay close to the bottom of the valley and the attraction force no longer has an advantage. However, FOCUS does not outperform Adam more than Signum did as the boundary between Adam and FOCUS is almost the same as that between Adam and Signum (\cref{fig:toy}, a and b). We conclude that FOCUS is an improved Signum that can outperform Adam when the effect of gradient noise is larger than that of sharpness.

\section{Experiments}\label{sec:exp}
\begin{figure*}[ht]
\vskip 0.2in
\begin{center}
\centerline{\includegraphics{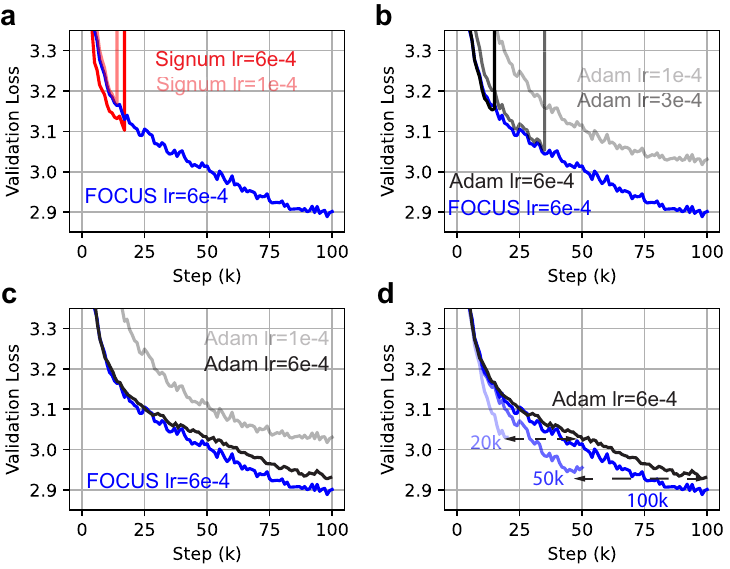}}
\caption{FOCUS is more stable and faster in training GPT-2 (small). (a) In \texttt{float16} and with the same hyperparameters, FOCUS is slower than Signum, but Signum is unstable. The smaller learning rate of Signum leads to slower training than FOCUS yet is still unstable. (b) Similarly, in \texttt{float16}, FOCUS is more stable than Adam and stable Adam training is much slower. (c) We copy the optimal performance of Adam (trained in \texttt{bfloat16}) from \citet{liu2023sophia} (black line), which is still slower than our FOCUS. (d) Decreasing the number of training steps of FOCUS (which also changes the learning rate scheduler), we find FOCUS can achieve a 2x speedup compared to Adam. More details in \cref{app:models}.}
\label{fig:exp}
\end{center}
\vskip -0.2in
\end{figure*}

\subsection{Setup}

Now, we are ready to evaluate whether our insights and proposed optimizer can contribute to LLM pre-training. We train a small GPT-2 model with 125 million parameters \cite{radford2019language,liu2023sophia} on OpenWebText \cite{gokaslan2019openwebtext}. Detailed model configuration can be found in \cref{app:gpt}.

We use the optimal performance of Adam (more precisely, AdamW \cite{loshchilov2018decoupled}) scanned over hyperparameters in \citet{liu2023sophia} as our baseline for comparison. Due to our limited computing resources, we cannot scan optimal performance for FOCUS systematically on GPT-2 training. For FOCUS, some hyperparameters around Adam's optimal choice, learning rate $6\times 10^{-4}$, $\beta_1=0.9$, $\beta_2=0.95$, and weight decay $0.1$, are tested. The idea is that if there exists any performance of FOCUS better than the best performance of Adam, the best performance of FOCUS is certainly better than that of Adam.

Our code implementation adopts those in \citet{liu2023sophia} to eight V100 GPUs with \texttt{float16}. We use a batch size of $480$. The learning rate is subject to a warm-up lasting two thousand steps and then a cosine schedule decreasing it to a final value around $0.05$ times the peak learning rate \cite{rae2021scaling}. If not specified, ``learning rate" in this section refers to the peak learning rate. A standard gradient clipping (by norm) is used with a threshold of $1.0$.

In the following, we show our results of GPT-2 training and compare them to our intuition from toy model experiments.

\subsection{Stability} 
Recall that FOCUS improves Signum when facing a sharp landscape with large gradient noise. We compare FOCUS ($\gamma = 0.2$, $\beta_1=0.9$, $\beta_2 = 0.99$, weight decay $0.2$) with Signum (i.e., FOCUS having $\gamma = 0$, $\beta_1=0.9$, $\beta_2 = 0$, and weight decay $0.2$). With the same learning rate $6\times 10^{-4}$ and the number of training steps $10^5$, we find FOCUS is slower than Signum at the early stages. If actual parameters are ``ahead" of the EMA of parameters, attraction force effectively decreases the step size from $1$ to $1-\gamma$. However, Signum rapidly becomes unstable (\cref{fig:exp}a). We then decrease the learning rate to $10^{-4}$ for Signum, yet still cannot obtain a stable training process. And at a learning rate $10^{-4}$, Signum is already slower than FOCUS. In training GPT-2, FOCUS outperforms Signum because FOCUS is much more stable and picks much larger learning rates.

We next compare FOCUS with Adam on our machines. We find the optimal hyperparameters reported in \citet{liu2023sophia} for Adam (i.e., learning rate $6\times 10^{-4}$, $\beta_1=0.9$, $\beta_2=0.95$, and weight decay $0.1$) lead to unstable training on V100 GPUs with \texttt{float16} (\cref{fig:exp}b). We doubled the weight decay but did not obtain stable training (see \cref{app:result}). Then, we decrease the learning rate to $3\times 10^{-4}$ and $10^{-4}$. Finally, Adam becomes stable at a learning rate $10^{-4}$ but yields much slower training and higher final validation loss than FOCUS (\cref{fig:exp}b). We conclude that FOCUS is more stable than Adam in \texttt{float16}.

\subsection{Speedup}

So far, we can claim FOCUS gains speedup compared to Signum and Adam on the machines we use. However, the speedup is mainly due to stability issues of other algorithms which may not be a problem on better machines or floating-point format (e.g., \texttt{bfloat16}) that are used widely.

To see whether FOCUS can have a more relevant speedup, we compare FOCUS with the optimal performance of Adam in \citet{liu2023sophia}. Adam is stable with a large learning rate $6\times 10^{-4}$ possibly due to the use of \texttt{bfloat16}. We observe that given the same training steps (i.e., $10^5$), one training case of FOCUS reaches a lower final loss than the optimal Adam reported in \citet{liu2023sophia} (\cref{fig:exp}c). If we decrease the training steps of FOCUS, which makes the scheduler decrease learning rate faster (the final learning rate remains the same), we find that FOCUS approximately achieves a twice speedup: the loss of FOCUS running $5\times 10^4$ steps is close to that of Adam running $10^5$ steps and the loss of FOCUS running $2\times 10^4$ steps is close to that of Adam running $5\times 10^4$ steps (\cref{fig:exp}d). Even though Adam can avoid instability issues via other techniques, FOCUS can still outperform Adam.

To conclude, our FOCUS optimizer, which aims to improve Signum to better handle sharpness and to avoid decreasing effective step size too much like Adam when gradient stochasticity is large, gains practical speedup in training GPT-2 compared to Adam. According to our toy model, this result suggests that LLM training has relatively large gradient noises, which may be a key limiting factor in training and needs more attention in future studies.

\section{Convergence analysis}\label{sec:theory}
As a sanity check, we analyze the worst-case convergence rate of FOCUS in the online convex optimization framework \cite{zinkevich2003online,kingma2014adammethodstochasticoptimization}. We consider the optimization of a sequence of convex loss functions $\{L_t\}_{t=1}^T$ over a convex set $\mathcal{F} \subseteq \mathbb{R}^d$.
\begin{definition}[Online Optimization Protocol]
\label{def:online}
At each time step $t$:
(1) The optimizer goes to parameters $\theta_t \in \mathcal{F}$;
(2) The loss function $L_t$ is revealed;
(3) The optimizer incurs loss $L_t(\theta_t)$ and observes gradient $g_t = \nabla L_t(\theta_t)$.
\end{definition}
The algorithm is evaluated via regret since the sequence is unknown in advance \cite{zinkevich2003online}.
The regret after $T$ iterations is
\begin{align}
    R(T) = \sum_{t=1}^T [L_t(\theta_t) - L_t(\theta^*)],
\end{align}
where $\theta^* = \mathrm{arg}\min_{\theta\in \mathcal{F}}\sum_{t=1}^TL_t(\theta)$.

Our analysis shows that FOCUS has a regret bound $O(\sqrt{T})$ which is the same as Adam in terms of dependence on the order of $T$ \cite{kingma2014adammethodstochasticoptimization}. In other words, the mean regret $R(T)/T$ converges as $O(1/\sqrt{T})$. The result is stated formally below and the proof is in \cref{app:theory}.
\begin{theorem}[Regret Bound]
\label{thm:main}
Let $\{L_t\}_{t=1}^T$ be a sequence of convex functions with bounded gradients $\|g_t\|_\infty \leq G_\infty$. For the FOCUS optimizer with learning rate $\eta_t = \eta/\sqrt{t}$, $\beta_1, \beta_2 \in [0,1)$, $\beta_{1,t} = \beta_1 \lambda_\beta^t$, $\gamma_t = \gamma \lambda_\gamma^t$,\footnote{We generalize the optimizer to have a time-dependent decay rate, i.e., $\beta_{1,t}$, and a time-dependent force strength, i.e., $\gamma_t$.} and $\lambda_\beta, \lambda_\gamma\in [0,1)$, assuming the distance between any two points in $\mathcal{F}$ are bounded in infinity norm by $D_\infty$ and in 2-norm by $D$, the regret is bounded by:
\begin{align}
    R(T) \leq &\frac{G_\infty(D^2 + d D_\infty^2\sqrt{T})}{2\eta(1-\beta_1)} + \frac{G_\infty \eta d \sqrt{T}}{(1-\beta_{1})}(1+|\gamma|)^2\nonumber \\
    + &\Big(\frac{\beta_{1}\lambda_\beta d}{1-\lambda_\beta}
    + \frac{|\gamma| \lambda_\gamma d}{1-\lambda_\gamma}\Big)\frac{G_\infty D_\infty}{1-\beta_{1}} 
\end{align}
\end{theorem}

Theoretically, decaying $\beta_{1,t}$ and $\gamma_t$ are important to bound some summation terms in $R(T)$. However, those summation terms may be small themselves and do not need decaying $\beta_{1,t}$ or $\gamma_t$. The worst-case analysis does not reflect the advantage of introducing momentum or attraction and has limited insights into designing new algorithms. On the other hand, it is still interesting to study the effect of decaying force. For a sanity check, we conclude that FOCUS can achieve the same convergence rate as Adam on the same online convex optimization problem.

\section{Related works}\label{sec:relate}

{\bf Valley-like loss landscapes} Although strongly convex loss landscapes are easier to analyze, empirical evidence has shown that loss landscapes are usually valley-like: \citet{sagun2016eigenvalues} found that eigenvalues of Hessians are split into the bulk part concentrated around zero and the edge part away from zero. The large eigenvalues correspond to directions with fast loss changes, while the near-zero eigenvalues correspond to the relatively flat regions at the bottom of the valley. This valley picture also agrees with the observation that gradient descent happens mostly in a tiny subspace~\cite{gur2018gradient}. Like our paper, research has looked into the role of noise in valley-like landscapes~\cite{wei2019noise}, suggesting that the noise drifts the parameter towards a less sharp landscape. Although noise in SGD was deemed beneficial in earlier works~\cite{kleinberg2018alternative,chaudhari2019entropy}, since they can help escape local minima and find flatter (hence more generalizable) solutions, recent work in large language models suggests that noise is the enemy. \citet{wen2024understanding} recently revisits the valley loss landscape in the context of language models, suggesting that noise or large learning rates prevent the optimizer from getting to the bottom of the valley. The main idea of FOCUS agrees with this picture: a good optimizer for LLMs should be able to squeeze into a sharp valley.

{\bf Optimizers} Despite the family of adaptive optimizers (Adam, Adagrad) dominating the deep learning world, there are newly proposed optimizers that are shown to demonstrate superior performance for language model pretraining. These methods include Lion~\cite{chen2024symbolic}, AdEMAMix~\cite{pagliardini2024ademamix}, MARS~\cite{yuan2024mars}, cautious optimizers~\cite{liang2024cautious}, modular optimizers~\cite{large2024scalable}, Sophia~\cite{liu2023sophia}, soap~\cite{vyas2024soap} (combining Adam with Shampoo~\cite{gupta2018shampoo}), muon~\cite{jordan2024muon}, Adam-mini~\cite{zhang2024adamminiusefewerlearning}, etc. These optimizers are mostly mathematics-inspired, while FOCUS is physics-inspired (self-attracting gas). Our analysis also involves physical tools like phase diagrams which illuminate when Adam/Signum is more performant than the other, providing insights into the question of why Signum works at all~\cite{bernstein2018signsgd,large2024scalable}, and that Adam is not always better than Signum~\cite{zhao2024deconstructing}.

{\bf Weight averaging} The EMA of a training trajectory has been used to obtain model parameters with better prediction stability and generalizability \cite{grill2020bootstraplatentnewapproach,He_2020_CVPR,Brotons2024exponential}. However, we use the EMA of parameters to modify the optimizer and change training dynamics directly.

\section{Discussion}\label{sec:disscuss}

We discuss our physics picture \cref{fig:outline}b in physics language. For a stochastic optimization process, the optimizer should end up being at the minimum of free energy
\begin{align}
    F = E -\Theta S,
\end{align}
where energy $E$ is the loss value, $\Theta$ is temperature related to batch size, learning rate, etc., and $S$ is entropy. Squeezing in the narrowing valley, the effective volume of the ``bouncing particles" is decreasing so that $S$ is decaying which tends to increase $F$. At some point, when the energetic force (i.e., gradient) balances the entropic force, the optimizer is stuck. We introduce attraction force which effectively adds a term to energy that prefers small volumes, fighting against the decaying entropy.

The mechanistic reasons behind LLM having sharp valley landscapes and large gradient stochasticity may be the same--large uncertainty in data distribution \cite{wen2024understanding}. Language is special in having such a large variability, diversity, and richness, giving rise to unique training dynamics. FOCUS outperforms Adam in our GPT-2 training (\cref{sec:exp}) because the batch size may still be small given the language task and dataset. Yet, how large the batch size we need for LLM to have a small gradient noise is an open question. 

The fact that FOCUS can be faster indicates large gradient noise is an important limiting factor in training LLMs. According to FOCUS's ability to deal with large noises and instabilities, we suggest trying it with limited batch sizes, low-precision training, etc. Future research is needed to fully explore the utility of FOCUS and our physical picture.

\section*{Acknowledgements}
Y. L. thanks Tongyang Li for helpful suggestions.


\bibliography{ref}
\bibliographystyle{icml2025}

\newpage
\appendix
\onecolumn

\section{Model details}\label{app:models}

\subsection{Toy model}\label{app:toy}

The experiment utilizes a two-dimensional optimization landscape defined by the function \cref{eq:toyloss}
where $u$ represents the coordinates after rotation by $\theta = \pi/6$ radians (30 degrees). The landscape parameters are set to $a=10$ and $c=0.1$. The rotation is implemented using a standard 2D rotation matrix:
\begin{equation}
    \begin{pmatrix} 
    \cos(\theta) & -\sin(\theta) \\
    \sin(\theta) & \cos(\theta)
    \end{pmatrix}
\end{equation}

The experiment compares three optimization configurations:
\begin{enumerate}
    \item Adam optimizer with $\beta_1=0.9$, $\beta_2=0.999$
    \item FOCUS optimizer with $\beta_1=0.9$, $\beta_2=0.9$, $\gamma=0.2$
    \item Signum optimizer with $\beta_1=0.9$
\end{enumerate}

The experiment shown in \cref{fig:toy}a systematically explores the following parameter ranges:
\begin{itemize}
\item Landscape sharpness ($a$): 96 values logarithmically spaced between $10^{-1}$ and $10^3$
\item Noise standard deviation: 100 values linearly spaced between 0 and 3
\item Learning rates: 20 values logarithmically spaced between $10^{-3}$ and $10^0$
\end{itemize}
Each configuration was tested with 50 independent replicates. For each replicate:
\begin{itemize}
\item Initial positions were randomly initialized with standard deviation $10^{-4}$
\item Optimization proceeded for 1000 steps
\item Weight decay was fixed at 0.1 across all experiments
\item The final loss value was recorded for analysis
\end{itemize}
The implementation includes several important technical considerations:
\begin{itemize}
\item Gradient computation incorporates multiplicative noise with a controlled standard deviation
\item Adam uses $\epsilon=10^{-8}$ for numerical stability in its denominator
\item Momentum-based bias correction is applied for both optimizers
\end{itemize}
Results are stored in a 6-dimensional array with the following dimensions:
\begin{itemize}
\item 3 (optimizer configurations)
\item 100 (noise levels)
\item 96 (sharpness)
\item 20 (learning rates)
\item 50 (replicates)
\item 1 (metrics: final loss)
\end{itemize}
Results for each landscape sharpness configuration are saved as separate numpy arrays, with filenames indicating the corresponding task ID in the format \texttt{syn-exp-7-\{task\_id\}.npy}.

In \cref{fig:toy}b, we do the following to obtain the $96$ pixel $\times$ $100$ pixel phase diagram. 
For each optimizer configuration, the following parameters were systematically varied:
\begin{itemize}
    \item Learning rates: 20 values logarithmically spaced between $10^{-3}$ and $10^0$
    \item Noise standard deviation: 100 values linearly spaced between 0 and 3
    \item Weight decay: 96 values linearly spaced between 0 and 0.5
\end{itemize}
Each configuration was tested with 50 independent replicates. For each replicate:
\begin{itemize}
    \item Initial positions were randomly initialized with standard deviation $10^{-4}$
    \item Optimization proceeded for 1000 steps
    \item Three metrics were tracked:
    \begin{enumerate}
        \item Final loss value
        \item Historical minimum loss achieved
        \item Projection onto the valley axis (defined by rotation angle $\pi/6$)
    \end{enumerate}
\end{itemize}
The implementation includes several important technical considerations:
\begin{itemize}
    \item Gradient computation incorporates multiplicative noise with a controlled standard deviation
    \item Weight decay is implemented as an additive term in parameter updates
    \item Adam uses $\epsilon=10^{-8}$ for numerical stability in its denominator
    \item Momentum-based bias correction is applied for both optimizers
\end{itemize}
The final results can be put in a 6-dimensional array with the following dimensions:
\begin{itemize}
    \item 3 (optimizer configurations)
    \item 100 (noise levels $\sigma$)
    \item 96 (weight decay)
    \item 20 (learning rates)
    \item 50 (replicates)
    \item 3 (metrics: final loss, minimum loss along the trajectory, valley projection)
\end{itemize}
Results for each weight decay configuration are saved as separate numpy arrays, with filenames indicating the corresponding task ID in the format \texttt{syn-exp-6-\{task\_id\}.npy}.

\subsection{MLP}\label{app:mlp}
The experiment in \cref{fig:toy} c and d employs a multi-layer perceptron (MLP) with the following architecture:
\begin{itemize}
    \item Input dimension: 784 (flattened MNIST images)
    \item Five hidden layers of 128 units each
    \item Output dimension: 10 (number of MNIST classes)
    \item ReLU activation functions between layers
    \item Batch normalization applied before each linear layer
\end{itemize}
The MNIST dataset was used with the following specifications:
\begin{itemize}
    \item Images resized to 28$\times$28 pixels
    \item Pixel values normalized to [0,1] through ToTensor transformation
    \item Training set: 60,000 images
    \item Test set: 10,000 images
    \item Test batch size fixed at 8,192 samples
\end{itemize}
The experiment compares four optimization configurations:
\begin{enumerate}
    \item AdamW with standard parameters ($\beta_1=0.9$, $\beta_2=0.999$)
    \item Signum with $\beta_1=0.9$
    \item FOCUS with $\beta_1=0.9$, $\beta_2=0.99$, $\gamma=0.2$
    \item FOCUS with $\beta_1=0.9$, $\beta_2=0.99$, $\gamma=0.4$
\end{enumerate}
The experiment systematically explores the following parameter ranges:
\begin{itemize}
    \item Learning rates: 20 values logarithmically spaced between $10^{-4}$ and $10^0$
    \item Batch sizes: 12 values as powers of 2, ranging from $2^1$ to $2^{12}$
    \item Weight decay: Fixed at $10^{-2}$ for all experiments
\end{itemize}
Each configuration was tested with three independent replicates. For each replicate:
\begin{itemize}
    \item Training proceeded for 400 steps
    \item Cross-entropy loss was used as the optimization objective
    \item Weight decay was selectively applied (excluded for bias and batch normalization parameters)
    \item Training loss, test loss, and test accuracy were recorded at each step
\end{itemize}
Several technical considerations were incorporated:
\begin{itemize}
    \item Gradient computations used PyTorch's autograd system
    \item Data loading utilized pin\_memory for improved GPU transfer efficiency
    \item Parameters were separated into two groups for differential weight decay application
    \item Test evaluation was performed with torch.no\_grad() for memory efficiency
\end{itemize}
All results are stored in a 5-dimensional tensor with the following dimensions:
\begin{itemize}
    \item 12 (batch sizes)
    \item 20 (learning rates)
    \item 3 (replicates)
    \item 400 (training steps)
    \item 3 (metrics: training loss, test loss, test accuracy)
\end{itemize}
Results for each combination of optimizer type and batch size are saved as separate PyTorch tensors, with filenames indicating the corresponding task ID in the format \texttt{scan-18-\{task\_id\}.pt}.

\subsection{GPT-2}\label{app:gpt}
We implement the GPT-2 small architecture with 125M parameters as described in \citet{radford2019language} which has embedding dimension $768$, number of heads $12$, and number of layers $12$.
We use \citet{liu2023sophia} codebase as our implementation foundation which is based on nanoGPT (\url{https://github.com/karpathy/nanoGPT/}). Following nanoGPT's architecture choices, we use GELU activations and disable bias and Dropout during pre-training.

Our model is trained on OpenWebText \citep{gokaslan2019openwebtext}. The text is tokenized using the standard GPT-2 tokenizer \citep{radford2019language}. We utilize the train and validation split provided by nanoGPT, where the training set contains 9B tokens and the validation set contains 4.4M tokens.

To enable efficient training, we implement distributed data parallel (DDP) training across 8 NVIDIA V100 GPUs with gradient accumulation to maintain an effective batch size of 480 ($480=12\times 5\times 8$, where $12$ is the batch size on each GPU, $5$ is the number of gradient accumulation steps, and $8$ is the number of GPUs). The entire training process is conducted using \texttt{float16} precision to optimize memory usage and computational efficiency.

\cref{fig:exp}a compares three optimization configurations:
\begin{enumerate}
    \item Signum with learning rate $6\times 10^{-4}$, $\beta_1=0.9$, weight decay $0.2$, $10^5$ training steps
    \item Signum with learning rate $10^{-4}$, $\beta_1=0.9$, weight decay $0.2$, $10^5$ training steps
    \item FOCUS with learning rate $6\times 10^{-4}$, $\beta_1=0.9$, $\beta_2=0.99$, $\gamma=0.2$, weight decay $0.2$, $10^5$ training steps
\end{enumerate}

\cref{fig:exp}b compares four optimization configurations:
\begin{enumerate}
    \item Adam with learning rate $6\times 10^{-4}$, $\beta_1=0.9$, $\beta_2=0.95$, weight decay $0.1$, $10^5$ training steps
    \item Adam with learning rate $3\times 10^{-4}$, $\beta_1=0.9$, $\beta_2=0.95$, weight decay $0.1$, $10^5$ training steps
    \item Adam with learning rate $10^{-4}$, $\beta_1=0.9$, $\beta_2=0.95$, weight decay $0.1$, $10^5$ training steps
    \item FOCUS with learning rate $6\times 10^{-4}$, $\beta_1=0.9$, $\beta_2=0.99$, $\gamma=0.2$, weight decay $0.2$, $10^5$ training steps
\end{enumerate}

\cref{fig:exp}c compares three optimization configurations:
\begin{enumerate}
    \item Adam with learning rate $6\times 10^{-4}$, $\beta_1=0.9$, $\beta_2=0.95$, weight decay $0.1$, $10^5$ training steps (result copied from \citet{liu2023sophia})
    \item Adam with learning rate $10^{-4}$, $\beta_1=0.9$, $\beta_2=0.95$, weight decay $0.1$, $10^5$ training steps
    \item FOCUS with learning rate $6\times 10^{-4}$, $\beta_1=0.9$, $\beta_2=0.99$, $\gamma=0.2$, weight decay $0.2$, $10^5$ training steps
\end{enumerate}

\cref{fig:exp}d compares four optimization configurations:
\begin{enumerate}
    \item Adam with learning rate $6\times 10^{-4}$, $\beta_1=0.9$, $\beta_2=0.95$, weight decay $0.1$, $10^5$ training steps (result copied from \citet{liu2023sophia})
    \item FOCUS with learning rate $6\times 10^{-4}$, $\beta_1=0.9$, $\beta_2=0.99$, $\gamma=0.2$, weight decay $0.2$, $10^5$ training steps
    \item FOCUS with learning rate $6\times 10^{-4}$, $\beta_1=0.9$, $\beta_2=0.99$, $\gamma=0.2$, weight decay $0.2$, $5\times 10^4$ training steps
    \item FOCUS with learning rate $10^{-3}$, $\beta_1=0.9$, $\beta_2=0.99$, $\gamma=0.2$, weight decay $0.2$, $2\times 10^4$ training steps
\end{enumerate}

\section{Supplementary results}\label{app:result}
\begin{figure*}
\vskip 0.2in
\begin{center}
\centerline{\includegraphics{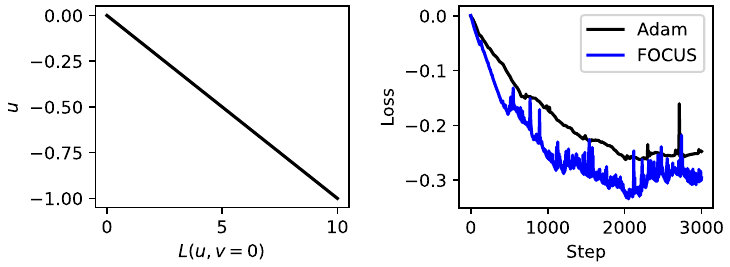}}
\caption{Training dynamics reaches steady state on the narrowing valley. The left panel is $L(u,v=0)=-cu$ showing our toy landscape can go to negative infinity ($c=0.1$ in this case). The right panel is the loss for Adam (learning rate $0.005$, $\beta_1=0.9$, $\beta_2=0.999$, weight decay $0.1$) and FOCUS (learning rate $0.005$, $\beta_1=0.9$, $\beta_2=0.9$, $\gamma=0.2$, weight decay $0.1$) on the toy landscape with $a=10$ and $c=0.1$. Both training dynamics already converge within $3000$ steps.}
\label{fig:satuation}
\end{center}
\vskip -0.2in
\end{figure*}

The landscape we construct can go to negative infinity (see \cref{fig:satuation} right panel) which is unrealistic. Our construction serves as a good local approximation. Despite the unrealistic part, we emphasize the key takeaway that because the valley is narrowing, the positions that can be reached with finite learning rates are far from the true minimum ($-\infty$ in our construction but $0$ in real cases). We also discuss this phenomenon through the balance between the gradient (energetic term) and an ``entropic" term associated with finite learning rate and gradient stochasticity (see \cref{sec:disscuss}). In practice, with non-zero weight decay, the optimizer can go along one parameter bounded by the inverse weight decay. However, the true parameter values the optimizers stay at (in \cref{fig:satuation} right panel, parameters are around $2\sim3$) are much smaller than the inverse weight decay (which is $10$ in \cref{fig:satuation}, the right panel). If the narrowing valley picture is true, we meet a problem for all optimizers similar to Zeno's paradox: we need a smaller learning rate to go deeper into the valley, but a smaller learning rate reduces the speed moving along the valley.
The narrowing valley makes the optimizer stay at a place far from the true optimum in the valley, which provides another perspective for the neural scaling laws \cite{kaplan2020scaling,hoffmann2022chinchilla}: the scaling laws may be due to training dynamics rather than the limit of the model.

In \cref{fig:pairs}, we provide detailed pair comparisons between optimizers based on the same data as \cref{fig:toy} a and b, which shows more details. As stated in the main text, the checkered pattern between blue and yellow may suggest that FOCUS does not reach significantly deeper in the valley but may be more stable such that with high probability we can see blue in the mixed region. We tend to interpret this checkered pattern through stochasticity rather than sensitivity to hyperparameters in deterministic training setup \cite{sohldickstein2024boundaryneuralnetworktrainability,liu2024complexfractaltrainabilityboundary}.  We also tune and generalize our toy landscapes which gives the same qualitative results (\cref{fig:moretoy}). For GPT-2 training, we find that increasing weight decay does not suffice to stabilize Adam (\cref{fig:weightd}), and therefore we decrease the learning rate in the main text to try to obtain stable training on our machines.
We also find that FOCUS needs a longer time to finish the same number of steps while the extra time is negligibly small (\cref{fig:walltime}). Since in terms of the algorithms, FOCUS is not more complex than Adam nor requires more memory, we hypothesize that this extra time is due to our unsophisticated code implementation of FOCUS which has some room for improvement.

\begin{figure*}
\vskip 0.2in
\begin{center}
\centerline{\includegraphics{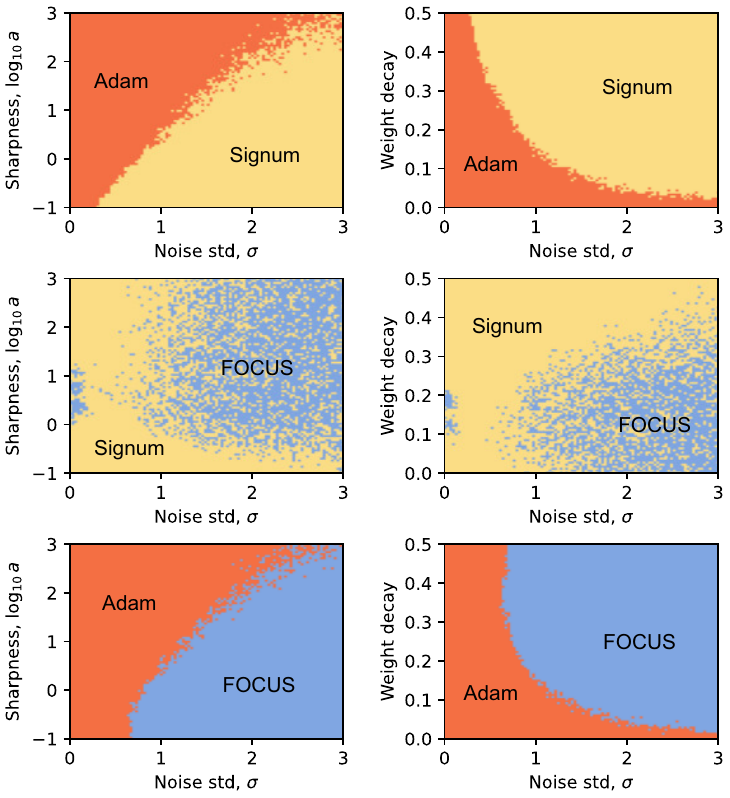}}
\caption{Pairwise comparisons between optimizers. This figure uses the same data as \cref{fig:toy}, a and b. We plot all the pairwise comparisons, showing more details to support the conclusion.}
\label{fig:pairs}
\end{center}
\vskip -0.2in
\end{figure*}

\begin{figure*}[ht]
\vskip 0.2in
\begin{center}
\centerline{\includegraphics{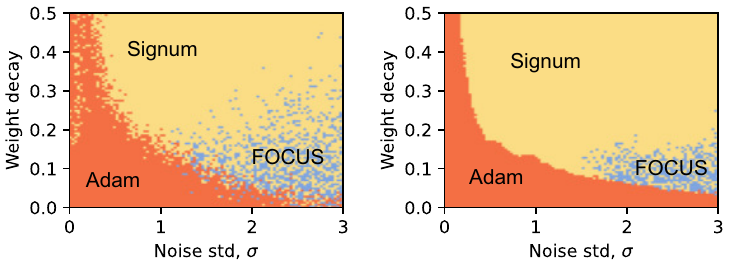}}
\caption{Toy landscape can have various forms and our results hold qualitatively and robustly. The left panel is obtained following the same procedure as \cref{fig:toy}b described in \cref{app:toy} but with $u$ being $x$ (i.e., no rotation). To obtain the right panel, we generalize the toy landscape to be high-dimensional ($100$ dimensions). The idea is to expand $v$ into a vector and the resulting landscape has axial symmetry along $u$. The $u,v$ coordinate is obtained by a random rotation from the parameter directions first.}
\label{fig:moretoy}
\end{center}
\vskip -0.2in
\end{figure*}

\begin{figure*}[ht]
\vskip 0.2in
\begin{center}
\centerline{\includegraphics{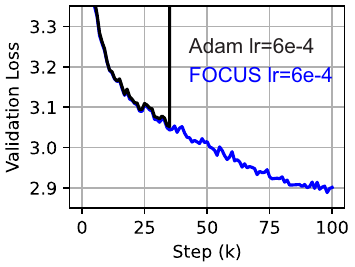}}
\caption{Increasing the weight decay cannot stabilize Adam. The black line here is Adam with learning rate $6\times 10^{-4}$, $\beta_1=0.9$, $\beta_2=0.95$, weight decay $0.2$, and $10^5$ training steps. The blue line is the same as \cref{fig:exp}.}
\label{fig:weightd}
\end{center}
\vskip -0.2in
\end{figure*}

\begin{figure*}[ht]
\vskip 0.2in
\begin{center}
\centerline{\includegraphics{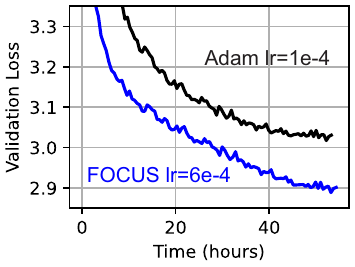}}
\caption{Adam and FOCUS have little difference in actual running time. We plot \cref{fig:exp}c in terms of wall time rather than steps and find FOCUS needs a negligibly longer time to run the same $10^5$ steps.}
\label{fig:walltime}
\end{center}
\vskip -0.2in
\end{figure*}

\section{Convergence proof}\label{app:theory}

\begin{lemma}[Convex function]
    \label{lem:convex}
    If a function $f$ over $\mathbb{R}^d$ is convex, then for all $x,y \in \mathbb{R}^d$,
    \begin{align}
        f(y) \geq f(x) + \langle \nabla f(x), y - x \rangle.
    \end{align}
\end{lemma}
The above lemma is a standard property of convex functions.

\begin{lemma}[Momentum bound]
    \label{lem:m}
    Let $m_{t,i}$ be the $i$th element of momentum vector $m_t$, $g_t$ is the gradient obtained at step $t$ satisfying $\| g_t\|_\infty \leq G_\infty$, and $m_t = \beta_{1,t} m_{t-1} + (1-\beta_{1,t}) g_t$ with $\beta_{1,t} \in [0,1)$, we have $|m_{t,i}|\leq G_\infty$ for all $t$ and $i$.
\end{lemma}
\begin{proof} 
First of all, $|m_0| = 0 \leq G_\infty$. We next assume that $|m_{t-1,i}|\leq G_\infty$ for all $i$. Then, $|m_{t,i}| = |\beta_{1,t} m_{t-1,i} + (1-\beta_{1,t}) g_{t,i}|\leq |\beta_{1,t} m_{t-1,i}| + |(1-\beta_{1,t}) g_{t,i}|\leq \beta_{1,t}G_\infty +(1-\beta_{1,t}) G_\infty = G_\infty$. According to mathematical induction, the proof is completed.
\end{proof}

\begin{theorem}[Regret Bound]
Let $\{L_t\}_{t=1}^T$ be a sequence of convex functions with bounded gradients $\|g_t\|_\infty \leq G_\infty$. For the FOCUS optimizer with learning rate $\eta_t = \eta/\sqrt{t}$, $\beta_1, \beta_2 \in [0,1)$, $\beta_{1,t} = \beta_1 \lambda_\beta^t$, $\gamma_t = \gamma \lambda_\gamma^t$, and $\lambda_\beta, \lambda_\gamma\in [0,1)$, assuming the distance between any two points in the feasible convex region $\mathcal{F}$ are bounded in infinity norm by $D_\infty$ and in 2-norm by $D$, the regret is bounded by:
\begin{align}
    R(T) \leq \frac{G_\infty(D^2 + d D_\infty^2\sqrt{T})}{2\eta(1-\beta_1)} 
    + \Big(\frac{\beta_{1}\lambda_\beta d}{1-\lambda_\beta}
    + \frac{|\gamma| \lambda_\gamma d}{1-\lambda_\gamma}\Big)\frac{G_\infty D_\infty}{1-\beta_{1}}
    + \frac{G_\infty \eta d \sqrt{T}}{(1-\beta_{1})}(1+|\gamma|)^2.
    \label{eq:rbound}
\end{align}
\end{theorem}
\begin{proof}
Recall the definition of regret,
\begin{align}
    R(T) = \sum_{t=1}^T [L_t(\theta_t) - L_t(\theta^*)].
\end{align}
We can use \cref{lem:convex} to estimate each term in $R(T)$,
\begin{align}
    L_t(\theta_t) - L_t(\theta^*) \leq \langle g_t , \theta_t - \theta^* \rangle = \sum_{i=1}^d g_{t,i}(\theta_{t,i} - \theta^*_{,i}).
\end{align}
The index $i$ above refers to the $i$th element of a vector. We therefore have $R(T) \leq \sum_{t=1}^T\sum_{i=1}^d g_{t,i}(\theta_{t,i} - \theta^*_{,i})$.

To study the term $g_{t,i}(\theta_{t,i} - \theta^*_{,i})$, we make use of FOCUS updating rule,
\begin{align}
    \theta_{t+1} = \theta_t - \eta_t (\mathrm{sign}(m_t) + \gamma_t \mathrm{sign}(\theta_t - \hat{\theta}_t)),
\end{align}
which can also be written as
\begin{align}
    \theta_{t+1,i} = \theta_{t,i} - \eta_t \left(\frac{m_{t,i}}{|m_{t,i}|} + \gamma_t \frac{\theta_{t,i} - \hat{\theta}_{t,i}}{|\theta_{t,i} - \hat{\theta}_{t,i}|}\right).
\end{align}
Subtracting $\theta^*$ both sides and taking square, we have
\begin{align}
    (\theta_{t+1,i} - \theta^*_{,i})^2 = &(\theta_{t,i} - \theta^*_{,i})^2 
    - 2\eta_t (\theta_{t,i} - \theta^*_{,i})\left(\frac{\beta_{1,t}m_{t-1,i} + (1-\beta_{1,t})g_{t,i}}{|m_{t,i}|} + \gamma_t \frac{\theta_{t,i} - \hat{\theta}_{t,i}}{|\theta_{t,i} - \hat{\theta}_{t,i}|}\right) \nonumber \\
    &+ \eta_t^2 \left(\frac{m_{t,i}}{|m_{t,i}|} + \gamma_t \frac{\theta_{t,i} - \hat{\theta}_{t,i}}{|\theta_{t,i} - \hat{\theta}_{t,i}|}\right)^2.
\end{align}
Rearranging the above equation, we can obtain
\begin{align}
    g_{t,i}(\theta_{t,i} - \theta^*_{,i}) = &\frac{|m_{t,i}|}{2\eta_t(1-\beta_{1,t})}[(\theta_{t,i} - \theta^*_{,i})^2 - (\theta_{t+1,i} - \theta^*_{,i})^2]
    -\frac{\beta_{1,t}m_{t-1,i}}{1-\beta_{1,t}}(\theta_{t,i} - \theta^*_{,i})\nonumber\\
    &-\frac{|m_{t,i}|\gamma_t}{1-\beta_{1,t}}(\theta_{t,i} - \theta^*_{,i})\frac{\theta_{t,i} - \hat{\theta}_{t,i}}{|\theta_{t,i} - \hat{\theta}_{t,i}|} 
    + \frac{|m_{t,i}|\eta_t}{2(1-\beta_{1,t})}\left(\frac{m_{t,i}}{|m_{t,i}|} + \gamma_t \frac{\theta_{t,i} - \hat{\theta}_{t,i}}{|\theta_{t,i} - \hat{\theta}_{t,i}|}\right)^2.
\end{align}
Note that $\beta_{1,t}\leq \beta_1$, $\gamma_{t}\leq \gamma$, and the last term contains a sum of two sign functions, we have
\begin{align}
    g_{t,i}(\theta_{t,i} - \theta^*_{,i}) \leq &\frac{|m_{t,i}|}{2\eta_t(1-\beta_{1})}[(\theta_{t,i} - \theta^*_{,i})^2 - (\theta_{t+1,i} - \theta^*_{,i})^2]
    -\frac{\beta_{1,t}m_{t-1,i}}{1-\beta_{1,t}}(\theta_{t,i} - \theta^*_{,i})\nonumber\\
    &-\frac{|m_{t,i}|\gamma_t}{1-\beta_{1,t}}(\theta_{t,i} - \theta^*_{,i})\frac{\theta_{t,i} - \hat{\theta}_{t,i}}{|\theta_{t,i} - \hat{\theta}_{t,i}|} 
    + \frac{|m_{t,i}|\eta_t}{2(1-\beta_{1,t})}(1+|\gamma|)^2
    \label{eq:innerp}
\end{align}
To obtain the bound for $R(T)$, we need to study the summation over $i$ and $T$ for the four terms at the RHS of \cref{eq:innerp}.

The summation of the first term at the RHS of \cref{eq:innerp} is
\begin{align}
    \sum_{i=1}^d\sum_{t=1}^T\frac{|m_{t,i}|}{2\eta_t(1-\beta_{1})}[(\theta_{t,i} - \theta^*_{,i})^2 - (\theta_{t+1,i} - \theta^*_{,i})^2] =& \sum_{i=1}^d \frac{|m_{1,i}|}{2\eta(1-\beta_{1})}(\theta_{1,i} - \theta^*_{,i})^2 \nonumber\\
    &+ \sum_{i=1}^d\sum_{t=2}^T\frac{1}{2(1-\beta_1)}(\theta_{t,i} - \theta^*_{,i})^2\left(\frac{|m_{t,i}|}{\eta_t} - \frac{|m_{t-1,i}|}{\eta_{t-1}}\right) \nonumber\\
    \leq & \frac{G_\infty D^2}{2\eta(1-\beta_{1})} + \sum_{i=1}^d \frac{1}{2\eta(1-\beta_1)}D_\infty^2G_\infty\sqrt{T} \nonumber\\
     = & \frac{G_\infty}{2\eta(1-\beta_1)} (D^2 + d D_\infty^2\sqrt{T}),
\end{align}
where we have used \cref{lem:m} to bound $|m_{t,i}|$ and other assumptions on gradient norm as well as parameter range.

Similarly, the summation of the second term at the RHS of \cref{eq:innerp} is
\begin{align}
    \sum_{i=1}^d\sum_{t=1}^T -\frac{\beta_{1,t}}{1-\beta_{1,t}}m_{t-1,i}(\theta_{t,i} - \theta^*_{,i})\leq \sum_{i=1}^d\sum_{t=1}^T \frac{\beta_{1}\lambda_\beta^t}{1-\beta_{1}}G_\infty D_\infty \leq \frac{\beta_{1}\lambda_\beta d}{(1-\beta_{1})(1-\lambda_\beta)}G_\infty D_\infty.
\end{align}

And the third term leads to
\begin{align}
    \sum_{i=1}^d\sum_{t=1}^T -\frac{|m_{t,i}|\gamma_t}{1-\beta_{1,t}}(\theta_{t,i} - \theta^*_{,i})\frac{\theta_{t,i} - \hat{\theta}_{t,i}}{|\theta_{t,i} - \hat{\theta}_{t,i}|} \leq \frac{|\gamma| \lambda_\gamma d}{(1-\beta_{1})(1-\lambda_\gamma)}G_\infty D_\infty.
\end{align}

The last term at the RHS of \cref{eq:innerp} yields
\begin{align}
    \sum_{i=1}^d\sum_{t=1}^T \frac{|m_{t,i}|\eta_t}{2(1-\beta_{1,t})}(1+|\gamma|)^2 \leq \sum_{i=1}^d \frac{G_\infty \eta}{2(1-\beta_{1})}(1+|\gamma|)^2 \sum_{t=1}^T\frac{1}{\sqrt{t}} 
    \leq \frac{G_\infty \eta d \sqrt{T}}{(1-\beta_{1})}(1+|\gamma|)^2.
\end{align}

Note that $R(T) \leq \sum_{t=1}^T\sum_{i=1}^d g_{t,i}(\theta_{t,i} - \theta^*_{,i})$, we combine the results from the four terms and can obtain \cref{eq:rbound}.
\end{proof}


\end{document}